%% file: rob_gen_arxiv.tex
\documentclass[twoside,11pt]{article}

\usepackage{jmlr2e}

\usepackage{amsmath,amssymb,enumerate}
\usepackage[dvips]{color}

\input{seteps.tex}

\input{setwmf.tex}

\newcommand{\pr}{\mathrm{Pr}}

\begin{document}
\date{}
\title{Robustness and Generalization}
\author{\name Huan Xu \email huan.xu@mail.utexas.edu
\\
\addr Department of Electrical and Computer Engineering \\
the University of Texas at Austin, TX, USA \AND \name Shie
Mannor \email shie@technion.ee.ac.il \\
\addr Department of  Electrical
 Engineering\\
Technion, Israel Institute of Technology}

\editor{n/a} \maketitle

\begin{abstract}
We derive generalization bounds for learning algorithms based on
their robustness: the property that if a testing sample is
``similar'' to a training sample, then the testing error is close to
the training error. This provides a novel approach, different from
the complexity or stability arguments, to study generalization
 of learning algorithms. We further show that a weak notion
of robustness is both sufficient and necessary for generalizability,
which implies that robustness is a fundamental property for learning
algorithms to work.
\end{abstract}

\section{Introduction}
The key issue in  the task of learning   from a set of observed
samples  is the estimation of the {\em risk} (i.e., generalization
error) of learning algorithms. Typically, its empirical measurement
(i.e., training error)  provides an optimistically biased
estimation, especially when the number of training samples is small.
Several approaches have been proposed to bound the deviation of the
risk from its empirical measurement, among which methods based on
uniform convergence
 and stability   are most widely used.

Uniform convergence of empirical quantities to their
mean~\citep[e.g.,][]{Vapnik74,Vapnik91}  provides ways to bound the
gap between the expected risk and the empirical risk
 by the complexity of the hypothesis set. Examples to complexity measures are the
Vapnik-Chervonenkis (VC) dimension
\citep[e.g.,][]{Vapnik91,Evgeniou00}, the fat-shattering dimension
\citep[e.g.,][]{AlonBendavidCesabianchiHaussler97,Bartlett98}, and
the Rademacher complexity \citep{Bartlett02,Bartlett05}. Another
well-known approach is based on {\em stability}. An algorithm is
stable if its output remains ``similar'' for different sets of
training samples that are identical up to removal or change of a
single sample. The first results that relate stability to
generalizability track back to \citet{DevroyeWagner79} and
\citet{DevroyeWagner79b}. Later, McDiarmid's \citep{McDiarmid89},
 concentration inequalities facilitated new bounds on
generalization error
\citep[e.g.,][]{Bousquet02,PoggioRifkinMukherjeeNiyogi04,MukherjeeNiyogiPoggioRifkin06}.

In this paper we explore a different approach which we term {\em
algorithmic robustness}. Briefly speaking, an algorithm is robust if
its solution has the following property: it
 achieves ``similar'' performance on a
testing sample and a training sample that are ``close''. This notion
of robustness is rooted in {\em robust optimization}
\citep{Ben-tal98,Ben-tal99,Bertsimas04} where a decision maker aims
to find a solution $x$ that minimizes a (parameterized) cost
function $f(x, \xi)$ with the knowledge that the unknown true
parameter $\xi$ may deviate from the observed parameter $\hat{\xi}$.
Hence, instead of solving $\min_x f(x, \hat{\xi})$ one solves
$\min_x [\max_{\tilde{\xi}\in \Delta}f(x, \tilde{\xi})]$, where
$\Delta$ includes all possible realizations of $\xi$. Robust
optimization was introduced in machine learning tasks to handle
exogenous noise
\citep[e.g.,][]{Bhattacharyya04b,Shivaswamy06,Globerson06}, i.e.,
the learning algorithm only has access to inaccurate observation of
training samples. Later on,
\citet{XuCaramanisMannorSVM1-08,XuCaramanisMannor-Lasso-NIPS} showed
that both Support Vector Machine(SVM) and Lasso have robust
optimization interpretation, i.e., they can be reformulated as
\[\min_{h\in \mathcal{H}} \max_{(\delta_1,\cdots,\delta_n) \in \Delta} \sum_{i=1}^n l(h, z_i+\delta_i),\]
for some $\Delta$. Here $z_i$ are the observed training samples and
$l(\cdot,\cdot)$ is the loss function (hinge-loss for SVM, and
squared loss for Lasso), which means that SVM and Lasso essentially
minimize the empirical error under the worst possible perturbation.
Indeed, as the authors of
\citet{XuCaramanisMannorSVM1-08,XuCaramanisMannor-Lasso-NIPS}
showed,   this reformulation leads to requiring that the loss of a
sample ``close'' to $z_i$ is small, which further implies
statistical consistency of these two algorithms. In this paper we
adopt this approach and study the (finite sample) generalization
ability of learning algorithms by investigating
 the loss of  learned hypotheses on samples
that   slightly deviate  from training samples.

Of special interest  is that robustness is more than just another
way to establish generalization bounds. Indeed, we show that a
weaker notion of robustness is a {\em necessary and sufficient}
condition of (asymptotic) generalizability of (general) learning
algorithms. While it is known    having a finite VC-dimension
\citep{Vapnik91} or equivalently being $\mathrm{CVEEE}_{loo}$ stable
\citep{MukherjeeNiyogiPoggioRifkin06} is necessary and sufficient
 for the Empirical Risk Minimization (ERM) to generalize, much less
 is known in the general case.
 Recently, \citet{huji} proposed a weaker notion
of stability that is necessary and sufficient for a learning
algorithm to be  consistent and generalizing, provided that the
problem itself is {\em learnable}. However, learnability requires
that the {\em convergence
 rate is uniform} with respect to all distributions, and is hence a
fairly strong assumption. In particular,   the standard supervised
learning setup where the hypothesis set is the set of measurable
functions is  {\em not} learnable since no algorithm can achieve a
uniform convergence rate  \citep[cf][]{DevroyeGyorfiLugosi96}.
Indeed,
 as the authors of \citet{huji} stated, for supervised learning problem
learnability is equivalent to the generalizability of ERM, and hence
reduce to the aforementioned results on ERM algorithms.

 In particular, our main contributions are the
following:
\begin{enumerate}
  \item We propose a notion of algorithmic robustness.
  Algorithmic robustness is a desired property for a learning algorithm since it implies a lack of sensitivity to (small) disturbances in the training data.
  \item Based on the notion of algorithmic robustness, we derive
  generalization bound for IID samples as well as samples drawn
  according to a Markovian chain.
  \item To
  illustrate the applicability of the notion of algorithmic
  robustness, we provide some examples of robust algorithms,
 including SVM, Lasso, feed-forward neural networks and
  PCA.
  \item We propose a weaker notion of robustness and show that it is
  both necessary and sufficient for a learning algorithm to
  generalize. This implies that robustness is an essential property
  needed for a learning algorithm to work.
\end{enumerate}

Note that while stability and robustness are similar on an intuitive
level, there is a difference between the two: stability requires
that nearly identical training sets with a single sample removed
lead to similar prediction rules, whereas robustness requires that a
prediction rule has comparable performance if tested on a sample
close to a training sample.

This paper is organized as follows. We define the notion of
robustness in Section~\ref{sec.define}, and prove generalization
bounds for robust algorithms in Section~\ref{sec.bounds}. In
Section~\ref{sec.pseudo} we propose a relaxed notion of robustness,
which is termed as pseudo-robustness, and show corresponding
generalization bounds. Examples of learning algorithms that are
robust or pseudo-robust are provided in Section~\ref{sec.example}.
Finally, we show that robustness is necessary and sufficient for
generalizability in Section~\ref{sec.equivalence}.
\subsection{Preliminaries}\label{sec.model}
We consider the
following general learning model:
 a set of  training samples are
given, and the goal
 is to pick a
hypothesis from a hypothesis set.  Unless otherwise mentioned,
throughout this paper the size of training set is fixed as $n$.
Therefore, we drop the dependence of parameters on the number of
training samples, while it should be understood that parameters may
vary with the number of training samples.
 We use
$\mathcal{Z}$ and $\mathcal{H}$ to denote the set from which each
sample is drawn, and
 the hypothesis set, respectively. Throughout the paper we use
 $\mathbf{s}$ to denote the training sample set consists of $n$
 training samples $(s_1,\cdots, s_n)$.
  A learning algorithm
 $\mathcal{A}$ is  thus a
mapping from $\mathcal{Z}^n$ to $\mathcal{H}$. We use
$\mathcal{A}_{\mathbf{s}}$ to represent the hypothesis learned
(given
 training  set $\mathbf{s}$). For each hypothesis $h\in \mathcal{H}$
 and a point $z\in \mathcal{Z}$, there is an associated loss $l(h,
 z)$. We ignore the issue of measurability
and further assume that $l(h,z)$ is non-negative and upper-bounded
uniformly by a scalar $M$.

In the special case of supervised learning, the sample space can be
decomposed as $\mathcal{Z}=\mathcal{Y}\times\mathcal{X}$, and the
goal is to learn a mapping  from $\mathcal{X}$ to $\mathcal{Y}$,
i.e., to predict the y-component given x-component. We hence use
$\mathcal{A}_{\mathbf{s}}(x)$ to represent the prediction of $x\in
\mathcal{X}$ if trained on $\mathbf{s}$.
 We call $\mathcal{X}$ the input space
and $\mathcal{Y}$ the output space. The output space can either be
$\mathcal{Y}=\{-1, +1\}$ for a classification problem, or
$\mathcal{Y}=\mathbb{R}$ for a regression problem.
  We use
$_{|x}$ and $_{|y}$ to denote the $x$-component and $y$-component of
a point. For example,
 $s_{i|x}$ is the $x$-component of $s_i$.
 To simplify notations, for a scaler $c$, we use $[c]^+$ to
 represent its non-negative part, i.e.,
  $[c]^+\triangleq
 \max(0,c)$.

We recall the following standard notion of covering number  from
\cite{Vaart2000}.
\begin{definition}[cf. \cite{Vaart2000}]For a metric space $S, \rho$ and
$T\subset S$ we say that $\hat{T}\subset S$ is an {\em
$\epsilon$-cover} of $T$, if $\forall t\in T$, $\exists \hat{t}\in
\hat{T}$ such that $\rho(t, \hat{t})\leq \epsilon$. The {\em
$\epsilon$-covering number} of $T$ is
\[\mathcal{N}(\epsilon, T, \rho)=\min\{|\hat{T}|\,: \hat{T} \mbox{ is an }\epsilon-\mbox{cover of }T\}.\]
\end{definition}
\section{Robustness of Learning Algorithms}\label{sec.define}
Before providing a precise definition of what we mean by
``robustness'' of an algorithm, we provide some motivating examples
which share a common property: if a testing sample is close to a
training sample, then the testing error is also close, a property we
will later formalize as ``robustness''.

We first consider  large-margin classifiers: Let the loss function
be $l(A_{\mathbf{s}}, z)= \mathbf{1}(A_{\mathbf{s}}(z_{|x})\not=
z_{|y} )$. Fix $\gamma>0$. An algorithm $\mathcal{A}_{\mathbf{s}}$
has a margin $\gamma$ if for $j=1,\cdots, n$
\[\mathcal{A}_{\mathbf{s}}(x)=\mathcal{A}_{\mathbf{s}}(s_{j|x});\quad
\forall x: \|x-s_{j|x}\|_2 < \gamma.
\] That is, any training sample is at least $\gamma$ away from the
classification boundary.

\begin{example}\label{exm.marginmoti}Fix $\gamma>0$ and put $K=2\mathcal{N}(\gamma/2, \mathcal{X}, \|\cdot\|_2)$.
If $\mathcal{A}_{\mathbf{s}}$ has a margin $\gamma$,  then
$\mathcal{Z}$ can be partitioned into $K$ disjoint sets, denoted by
$\{C_i\}_{i=1}^K$, such that if $s_j$ and $z\in \mathcal{Z}$ belong
to a same $C_i$, then $|l(\mathcal{A}_{\mathbf{s}},
s_j)-l(\mathcal{A}_{\mathbf{s}}, z)|=0$.
\end{example}
\begin{proof}By definition of covering number, we can partition
$\mathcal{X}$ into $\mathcal{N}(\gamma/2, \mathcal{X}, \|\cdot\|_2)$
subsets (denoted $\hat{X}_i$) such that each subset has a diameter
less or equal to $\gamma$. Further, $\mathcal{Y}$ can be partitioned
to $\{-1\}$ and $\{+1\}$. Thus, we can partition $\mathcal{Z}$ into
$2\mathcal{N}(\gamma/2, \mathcal{X}, \|\cdot\|_2)$ subsets such that
if $z_1, z_2$ belong to a same subset, then $y_{1|y}=y_{2|y}$ and
$\|x_{1|y}-x_{2|y}\|\leq \gamma$. By definition of margin, this
guarantees that if $s_j$ and $z\in \mathcal{Z}$ belong to a same
$C_i$, then $|l(\mathcal{A}_{\mathbf{s}},
s_j)-l(\mathcal{A}_{\mathbf{s}}, z)|=0$.
\end{proof}

The next example is a linear regression algorithm. Let the loss
function be  $l(A_{\mathbf{s}}, z)= |z_{|y}-
A_{\mathbf{s}}(z_{|x})|$, and let
 $\mathcal{X}$
be a bounded subset of $\mathbb{R}^m$ and fix $c>0$. The
norm-constrained linear regression algorithm is
\begin{equation}\label{equ.NCLR}
\begin{split}&\mathcal{A}_{\mathbf{s}} =\min_{w\in \mathbb{R}^m: \|w\|_2 \leq c} \sum_{i=1}^n |s_{i|y}-w^\top
s_{i|x}|,
\end{split}
\end{equation}
i.e., minimizing the empirical error among all linear classifiers
whose norm is bounded.
\begin{example}Fix $\epsilon>0$ and put $K=\mathcal{N}(\epsilon/2, \mathcal{X}, \|\cdot\|_2)\times \mathcal{N}(\epsilon/2,\mathcal{Y}, |\cdot|)$. Consider the algorithm
as in~(\ref{equ.NCLR}).  The set $\mathcal{Z}$ can be partitioned
into $K$ disjoint sets, such that if $s_j$ and $z\in \mathcal{Z}$
belong to a same $C_i$, then
\[|l(\mathcal{A}_{\mathbf{s}}, s_j)-l(\mathcal{A}_{\mathbf{s}},
z)|\leq (c+1)\epsilon.\]
\end{example}
\begin{proof}Similarly to the previous example, we can partition
$\mathcal{Z}$ to $\mathcal{N}(\epsilon/2, \mathcal{X},
\|\cdot\|_2)\times \mathcal{N}(\epsilon/2,\mathcal{Y}, |\cdot|)$
subsets, such that if $z_1, z_2$ belong to a same $C_i$, then
$\|z_{1|x}-z_{2|x}\|_2\leq \epsilon$, and $|z_{1|y}-z_{2|y}| \leq
\epsilon$. Since $\|w\|_2 \leq c$, we have
\begin{equation*}\begin{split}
\left|l(w  , z_1)-l(w (\mathbf{s}), z_2)\right| =&\left||z_{1|y}
-w^\top z_{1|x}|-|z_{2|y} -w^\top
z_{2|x}|\right|\\
\leq &\left|(z_{1|y} -w^\top z_{1|x})-(z_{2|y} -w^\top
z_{2|x})\right|\\
\leq &|z_{1|y}-z_{2|y}|+\|w\|_2
\|z_{1|x}-z_{2|x}\|_2\\
\leq & (1+c)\epsilon,
\end{split}\end{equation*}whenever $z_1, z_2$ belong to a same
$C_i$.
\end{proof}

The two motivating examples both share a property: we can partition
the sample set into finite subsets, such that if a new sample falls
into the same subset as a testing sample, then the loss of the
former is close to the loss of the latter. We call an algorithm
having this property ``robust.''

\begin{definition}\label{def.robustalgorithm}
Algorithm $\mathcal{A}$ is $(K,\, \epsilon(\mathbf{s}))$ robust if
 $\mathcal{Z}$ can be partitioned into $K$ disjoint
sets, denoted as $\{C_i\}_{i=1}^K$, such that  $\forall s\in
\mathbf{s}$,
\begin{equation}\label{equ.robust}s, z
\in C_i,\quad \Longrightarrow \quad \left|l(\mathcal{A}_\mathbf{s},
s)-l(\mathcal{A}_\mathbf{s}, z)\right| \leq
\epsilon(\mathbf{s}).\end{equation}
\end{definition}
In the definition, both $K$ and the partition sets $\{C_i\}_{i=1}^K$
do not depend on the training set $\mathbf{s}$. Note that the
definition of robustness requires that~(\ref{equ.robust}) holds {\em
for
 every} training
sample.
 Indeed, we can relax the definition, so that
the condition needs only hold for a subset of training samples. We
call an algorithm having this property ``pseudo robust''. See
Section~\ref{sec.pseudo} for details.

\section{Generalization of Robust
Algorithms}\label{sec.bounds} In this section we investigate
generalization property of robust algorithms. In particular, in the
following subsections we derive PAC bounds  for robust algorithms
under three different conditions: (1) The ubiquitous learning setup
where the samples are i.i.d.\ and the goal of learning is to
minimize expected loss. (2) The learning goal is to minimize
quantile loss. (3) The samples are generated according to a
(Doeblin) Markovian chain.   Indeed, the fact that we can provide
results in (2) and (3) indicates the fundamental nature of
robustness as a property of learning algorithms.
\subsection{IID samples and expected loss}
In this section, we consider the standard learning setup, i.e., the
sample set $\mathbf{s}$ consists of $n$ i.i.d.\ samples generated by
an unknown distribution $\mu$, and the goal of learning is to
minimize expected test loss. Let $\hat{l}(\cdot)$ and
$l_{\mathrm{emp}}(\cdot)$ denote the expected error and the training
error, i.e.,
\[\hat{l}(\mathcal{A}_\mathbf{s})\triangleq \mathbb{E}_{z\sim \mu} l(\mathcal{A}_\mathbf{s}, z);\quad l_{\mathrm{emp}}(\mathcal{A}_\mathbf{s}) \triangleq \frac{1}{n}\sum_{s_i\in \mathbf{s}} l(\mathcal{A}_{\mathbf{s}}, s_i).\]
 Recall that the loss function $l(\cdot,\cdot)$ is upper bounded by
$M$.
\begin{theorem}\label{thm.main} If $\mathbf{s}$ consists of $n$ i.i.d. samples, and  $\mathcal{A}$
is $(K, \epsilon(\mathbf{s}))$-robust, then for any $\delta>0$, with
probability at least $1-\delta$,
\[\left|\hat{l}(\mathcal{A}_\mathbf{s})-l_{\mathrm{emp}}(\mathcal{A}_{\mathbf{s}})\right|\leq \epsilon(s)+M
\sqrt{\frac{2K\ln 2 + 2\ln(1/\delta)}{n}}. \]
\end{theorem}
\begin{proof}
 Let $N_i$  be the set of index of points of
 $\mathbf{s}$
 that fall into the $C_i$.
  Note that  $(|N_1|, \cdots,
|N_K|)$  is an IID multinomial random variable with parameters $n$
and $(\mu(C_1),\cdots, \mu(C_K))$. The following holds by the
Breteganolle-Huber-Carol inequality \citep[cf Proposition A6.6
of][]{Vaart2000}:
\[\mathrm{Pr}\left\{\sum_{i=1}^{K}\left|\frac{|N_i|}{n} -\mu(C_i)\right| \geq \lambda \right\}\leq 2^{K}\exp (\frac{-n\lambda^2}{2}).\]
Hence, the following holds with probability at least $1-\delta$,
\begin{equation}\label{equ.proofmain}\sum_{i=1}^{K}\left|\frac{|N_i|}{n} -\mu(C_i)\right| \leq
\sqrt{\frac{2K\ln 2 + 2\ln(1/\delta)}{n}}.\end{equation}

We have
\begin{equation}\label{equ.proofofmain}\begin{split}&\left|\hat{l}(\mathcal{A}_\mathbf{s})-l_{\mathrm{emp}}
(\mathcal{A}_{\mathbf{s}})\right|
\\= &\left|\sum_{i=1}^K \mathbb{E} \big(l(\mathcal{A}_\mathbf{s},
z)|z\in C_i\big)\mu(C_i) -\frac{1}{n}\sum_{i=1}^n  l(\mathcal{A}_{\mathbf{s}}, s_i) \right|\\
\stackrel{(a)}{\leq}  &\left|\sum_{i=1}^K \mathbb{E}
\big(l(\mathcal{A}_\mathbf{s}, z)|z\in
C_i\big)\frac{|N_i|}{n}-\frac{1}{n}\sum_{i=1}^n
l(\mathcal{A}_{\mathbf{s}}, s_i) \right|\\&\qquad+\left|\sum_{i=1}^K
\mathbb{E} \big(l(\mathcal{A}_\mathbf{s}, z)|z\in C_i\big)\mu(C_i)
-\sum_{i=1}^K \mathbb{E} \big(l(\mathcal{A}_\mathbf{s}, z)|z\in
C_i\big)\frac{|N_i|}{n}\right|\\
\stackrel{(b)}{\leq} & \left|\frac{1}{n} \sum_{i=1}^K\sum_{j\in
N_i}\max_{z_2\in C_i} |l(\mathcal{A}_\mathbf{s},
s_j)-l(\mathcal{A}_\mathbf{s}, z_2)|\right| +\left|\max_{z\in
\mathcal{Z}} |l(\mathcal{A}_{\mathbf{s},
z})|\sum_{i=1}^K\Big|\frac{|N_i|}{n}-\mu(C_i)\Big| \right|\\
\stackrel{(c)}{\leq} &  \epsilon(s)+   M
\sum_{i=1}^{K}\left|\frac{|N_i|}{n} -\mu(C_i)\right|,
\end{split}\end{equation}
 where (a), (b), and (c) are due to the triangle
inequality, the definition of $N_i$, and the definition of
$\epsilon(\mathbf{s})$ and $M$, respectively. Note that the
right-hand-side of~(\ref{equ.proofofmain}) is upper-bounded by
$\epsilon(s)+M \sqrt{\frac{2K\ln 2 + 2\ln(1/\delta)}{n}}$ with
probability at least $1-\delta$ due to~(\ref{equ.proofmain}). The
theorem follows.
\end{proof}

Theorem~\ref{thm.main} requires that we fix a $K$ {\em a priori}.
However,   it is often worthwhile to consider adaptive $K$. For
example, in the large-margin classification case, typically the
margin is known only after $\mathbf{s}$ is realized. That is, the
value of $K$ depends on $\mathbf{s}$. Because of this dependency, we
needs a generalization bound that holds uniformly for all $K$.

\begin{corollary}\label{cor.uniform}If $\mathbf{s}$ consists of $n$ i.i.d. samples,
and  $\mathcal{A}$ is $(K, \epsilon_K(\mathbf{s}))$  robust for all
$K\geq 1$, then for any $\delta>0$, with probability at least
$1-\delta$,
\[\left|\hat{l}(\mathcal{A}_\mathbf{s})-l_{\mathrm{emp}}(\mathcal{A}_{\mathbf{s}})\right|\leq
\inf_{K \geq 1} \left[\epsilon_K(s)+M  \sqrt{\frac{2K\ln 2 +
2\ln\frac{K(K+1)}{\delta}}{n}} \right].
\]
\end{corollary}
\begin{proof}Let \[E(K)\triangleq\left\{
\left|\hat{l}(\mathcal{A}_\mathbf{s})-l_{\mathrm{emp}}(\mathcal{A}_{\mathbf{s}})\right|>
 \epsilon_K(s)+M \sqrt{\frac{2K\ln 2 +
2\ln\frac{K(K+1)}{\delta}}{n}} \right\}.\] From
Theorem~\ref{thm.main} we have $\mathrm{Pr}(E(K)) \leq
\delta/(K(K+1)) =\delta/K -\delta/(K+1)$. By the union bound we have
\[\mathrm{Pr}\left\{\bigcup_{K\geq 1} E(K) \right\}\leq \sum_{K\geq 1} \mathrm{Pr}\left( E(K)\right)\leq \sum_{K\geq 1} \left[\frac{\delta}{K}-\frac{\delta}{K+1}\right] =\delta, \]
and the corollary follows.
\end{proof}
  If $\epsilon(s)$ does not depend on $\mathbf{s}$,
we can sharpen the bound given in Corollary~\ref{cor.uniform}.
\begin{corollary}\label{cor.sharpuniform}If $\mathbf{s}$ consists of $n$ i.i.d. samples, and
 $\mathcal{A}$ is $(K, \epsilon_K)$ robust for all $K\geq 1$,
then for any $\delta>0$, with probability at least $1-\delta$,
\[\left|\hat{l}(\mathcal{A}_\mathbf{s})-l_{\mathrm{emp}}(\mathcal{A}_{\mathbf{s}})\right|\leq \inf_{K \geq 1}
\left[ \epsilon_K+M \sqrt{\frac{2K\ln 2 + 2\ln\frac{1}{\delta}}{n}}
\right]. \]
\end{corollary}
\begin{proof} The right hand side does not depend on $\mathbf{s}$,
and hence the optimal $K^*$. Therefore, plugging $K^*$ into
Theorem~\ref{thm.main} establishes the corollary.
\end{proof}

\subsection{Quantile Loss}\label{sss.quantile}
So far we considered the standard expected loss setup. In this
section we consider some less extensively investigated loss
functions, namely quantile value and truncated expectation (see the
following for precise definitions). These loss functions are of
interest because they are less sensitive to  the presence of
outliers than the standard average loss \citep{Huber81}.

\begin{definition}\label{def.robust}For a non-negative random variable $X$, the {\em $\beta$-quantile value} is
\[\mathbb{Q}^\beta(X)\triangleq  \inf\left\{c\in
\mathbb{R}: \mathrm{Pr} \big(X\leq c\big) \geq \beta\right\}.\] The
{\em $\beta$-truncated mean} is
\[\mathbb{T}^\beta(X)\triangleq \left\{ \begin{array}{ll}
\mathbb{E}\left[X\cdot\mathbf{1}(X < \mathbb{Q}^\beta(X))\right] &
\mbox{if}\,\, \mathrm{Pr}\big[X=\mathbb{Q}^\beta(X)\big]=0;\\
 \mathbb{E}\left[X\cdot\mathbf{1}(X < \mathbb{Q}^\beta(X))\right] +\frac{\beta-\mathrm{Pr}\big[X<\mathbb{Q}^\beta(X)\big]}{\mathrm{Pr}\big[X=\mathbb{Q}^\beta(X)\big]}\mathbb{Q}^\beta(X) &\mbox{otherwise}.\end{array}\right.\]
\end{definition}
In words, the $\beta-$quantile loss is the smallest value that is
larger or equal to $X$ with probability at least $\beta$. The
$\beta$-truncated mean is the contribution to the expectation of the
leftmost $\beta$ fraction of the distribution. For example, suppose
$X$ is supported on $\{c_1,\cdots, c_{10}\}$
($c_1<c_2<\cdots<c_{10}$) and the probability of taking each value
equals $0.1$. Then the $0.63$-quantile loss of $X$ is $c_7$, and the
$0.63$-truncated mean of $X$ equals $0.1(\sum_{i=1}^6 c_i +0.3
c_7)$.

Given $h\in \mathcal{H}$, $\beta\in (0,\,1)$, and a probability
measure $\mu$ on $\mathcal{Z}$,  let
\[\mathcal{Q}(h, \beta, \mu)\triangleq \mathbb{Q}^{\beta}(l(h, z));\quad \mbox{where:}\,\, z\sim \mu;\]
 and
 \[\mathcal{T}(h, \beta, \mu)\triangleq \mathbb{T}^{\beta}(l(h, z));\quad \mbox{where:}\,\, z\sim \mu;\]
i.e., the $\beta$-quantile value and $\beta$-truncated mean of the
(random) testing error of hypothesis $h$ if the testing sample
follows distribution $\mu$.
 We have the following theorem that is a special case of
Theorem~\ref{thm.pseudoquantile}, hence we omit the proof.
\begin{theorem}[Quantile Value \& Truncated
Mean]\label{thm.quantile} Suppose $\mathbf{s}$ are $n$ i.i.d.
samples drawn according to $\mu$, and denote the empirical
distribution  of $\mathbf{s}$ by $\mu_{\mathrm{emp}}$. Let
$\lambda_0 =\sqrt{\frac{2K\ln 2 + 2\ln(1/\delta)}{n}}$. If $0\leq
\beta-\lambda_0\leq \beta+\lambda_0\leq 1$ and $\mathcal{A}$ is $(K,
\epsilon(\mathbf{s}))$ robust, then with probability at least
$1-\delta$, the followings hold
\begin{equation*}\begin{split}(I)\quad &\mathcal{Q}\left(\mathcal{A}_{\mathbf{s}},
\beta-\lambda_0,\mu_{\mathrm{emp}}\right) -\epsilon(\mathbf{s}) \leq
\mathcal{Q}\left(\mathcal{A}_{\mathbf{s}}, \beta,\mu\right)
 \leq
\mathcal{Q}\left(\mathcal{A}_{\mathbf{s}}, \beta+\lambda_0
,\mu_{\mathrm{emp}}\right)
+\epsilon(\mathbf{s});\\
(II)\quad &\mathcal{T}\left(\mathcal{A}_{\mathbf{s}},
\beta-\lambda_0,\mu_{\mathrm{emp}}\right) -\epsilon(\mathbf{s}) \leq
\mathcal{T}\left(\mathcal{A}_{\mathbf{s}}, \beta,\mu\right)
 \leq
\mathcal{T}\left(\mathcal{A}_{\mathbf{s}}, \beta+\lambda_0
,\mu_{\mathrm{emp}}\right) +\epsilon(\mathbf{s}).
\end{split}\end{equation*}
\end{theorem}
In words, Theorem~\ref{thm.quantile} essentially means that with
high probability, the $\beta$-quantile value/truncated mean of the
testing error (recall that the testing error is a random variable)
is (approximately) bounded by the $(\beta\pm\lambda_0)$-quantile
value/truncated mean of the empirical error, thus providing a way to
estimate the quantile value/truncated expectation of the testing
error based on empirical observations.

\subsection{Markovian samples}\label{sec.Markovian} The robustness approach is not
restricted to the IID setup. In many applications of interest, such
as reinforcement learning and time series forecasting, the IID
assumption is violated. In such applications there is a time driven
process that generates samples that depend on the previous samples
(e.g., the observations of a trajectory of a robot). Such a
situation can be modeled by stochastic process such as a Markov
processes. In this section we establish similar result to the IID
case for samples that are drawn from a Markov chain.
 The state space can be general, i.e.,
it is not necessarily finite or countable. Thus, a certain ergodic
structure of the underlying Markov chain is needed. We focus on
chains that converge to equilibrium exponentially fast and uniformly
in the initial condition. It is known
 that this is equivalent to the class of of
Doeblin chains \citep{MeynTweedie93}. Recall the following
definition~\citep[cf][]{MeynTweedie93,Doob53}).
\begin{definition}\label{def.doeblin}
A Markov chain $\{z_i\}_{i=1}^{\infty}$ on a state space
$\mathcal{Z}$  is a {\em Doeblin chain} (with $\alpha$ and $T$) if
there exists a probability measure $\varphi$ on $\mathcal{Z}$,
$\alpha>0$, an integer $T\geq 1$ such that
\[\mathrm{Pr}(z_T\in H| z_0=z)\geq \alpha \varphi(H); \,\, \forall\, \mbox{measureable}\,\, H\subseteq \mathcal{Z};\,\,\forall z\in \mathcal{Z}.\]
\end{definition}
The class of Doeblin chains is probably the ``nicest'' class of
general state-space Markov chains. We notice that such assumption is
not overly restrictive, since by requiring that an ergodic theorem
holds for all bounded functions uniformly in the initial
distribution itself implies that a chain is Doeblin
\citep{MeynTweedie93}. In particular, an ergodic chain defined on  a
finite state-space is a Doeblin chain.

Indeed, the Doeblin chain condition guarantees that an invariant
measure $\pi$ exists. Furthermore, we have the
 following lemma adapted from Theorem 2
of~\citet{GlynnOrmoneit02}.
\begin{lemma}\label{lem.doeblin} Let  $\{z_i\}$ be a  Doeblin  chain
as in Definition~\ref{def.doeblin}. Fix a function $f:
\mathcal{Z}\rightarrow \mathbb{R}$ such that $\|f\|_{\infty} \leq
C$. Then  for $n> 2C T/\epsilon \alpha$ the following holds
\begin{equation*}\begin{split}&\mathrm{Pr}\left(\frac{1}{n}\sum_{i=1}^nf(z_i) -\int_{\mathcal{Z}} f(z) \pi(dz) s\geq  \epsilon\right)  \leq
\exp\left(-\frac{\alpha^2(n\epsilon-2CT/\alpha)^2}{2n C^2
T^2}\right).\end{split}\end{equation*}
\end{lemma}
The following is the main theorem of this section that establishes a
generalization bound for robust algorithms with samples drawn
according to a Doeblin chain.
\begin{theorem}
Let $\mathbf{s}=\{s_1,\cdots, s_n\}$ be the first $n$ outputs
 of a Doeblin chain with $\alpha$ and $T$ such that
$n>2T/\alpha$,
 and suppose that  $\mathcal{A}$
is $(K, \epsilon(\mathbf{s}))$-robust. Then for any $\delta>0$, with
probability at least $1-\delta$,
\[\left|\hat{l}(\mathcal{A}_\mathbf{s})-l_{\mathrm{emp}}(\mathcal{A}_{\mathbf{s}})\right|\leq
\epsilon(s)+M\left(\frac{8T^2(K \ln2 +\ln (1/\delta))}{\alpha^2 n}
\right)^{1/4}.
\]
\end{theorem}
\begin{proof}
We prove the following slightly stronger statement:
\begin{equation}\label{equ.proofmarkovian}\left|\hat{l}(\mathcal{A}_\mathbf{s})-l_{\mathrm{emp}}(\mathcal{A}_{\mathbf{s}})\right|\leq
\epsilon(s)+M\sqrt{\frac{T}{\alpha n}} \sqrt{\sqrt{2n(K \ln2 +\ln
(1/\delta))} +2}.\end{equation}
 Let
$\lambda_0=\sqrt{\frac{T}{\alpha n}} \sqrt{\sqrt{2n(K \ln2 +\ln
(1/\delta))} +2}$, we have that $\lambda_0> \sqrt{2T/\alpha n}$.
Since $n> 2T/\alpha$, we have $n> \sqrt{2Tn/\alpha}$, which  leads
to
\[n>\frac{2T}{\alpha\sqrt{2T/\alpha n}}>\frac{2T}{\alpha
\lambda_0}.\]

Let $N_i$  be the set of index of points of
 $\mathbf{s}$
 that fall into the $C_i$.
Consider the set of functions
$\mathcal{H}=\{\mathbf{1}(\mathbf{x}\in H)| H=\bigcup_{i\in I}
C_i;\,\, \forall I \subseteq \{1,\cdots, K\}\}$, i.e., the set of
indicator functions of all different unions of $C_i$. Then
$|\mathcal{H}|=2^K$. Furthermore, fix a $h_0\in \mathcal{H}$,
\begin{equation*}
\begin{split}&\mathrm{Pr}(\sum_{j=1}^{K}\left|\frac{|N_j|}{n}-\pi(C_j)\right|
\geq \lambda)\\ =&\mathrm{Pr}\Big\{\sup_{h\in \mathcal{H}}
[\frac{1}{n}\sum_{i=1}^n h(s_i)- \mathbb{E}_{\pi} h(s)] \geq\lambda
\Big\}
\\ \leq & 2^{K}
 \mathrm{Pr}[\frac{1}{n}\sum_{i=1}^n h_0(s_i)- \mathbb{E}_{\pi} h_0(s)\geq \lambda].
 \end{split}
 \end{equation*}
Since $\|h_0\|_{\infty}=1$, we can apply Lemma~\ref{lem.doeblin} to
get for  $n>  2 T/\lambda \alpha$
\[ \mathrm{Pr}[\frac{1}{n}\sum_{i=1}^n h_0(s_i)- \mathbb{E}_{\pi}
h_0(s)\geq \lambda] \leq \exp\left(-\frac{\alpha^2 (n
\lambda^2-2T/\alpha)^2} {2nT^2}\right).\] Substitute in $\lambda_0$,
\[\mathrm{Pr}(\sum_{j=1}^{K}\left|\frac{|N_j|}{n}-\pi(C_j)\right|
\geq \lambda_0) \leq  2^K\exp\left(-\frac{\alpha^2 (n
\lambda_0^2-2T/\alpha)^2} {2nT^2}\right)=\delta.\]  Thus,
(\ref{equ.proofmarkovian}) follows by an identical argument as the
proof of Theorem~\ref{thm.main}.

To complete the proof of the theorem, note that $n>2T/\alpha$
implies $n\geq 2$, hence $\sqrt{2n(K \ln2 +\ln (1/\delta))}\geq 2$.
Therefore, \begin{equation*}\begin{split}\sqrt{\frac{T}{\alpha n}}
\sqrt{\sqrt{2n(K \ln2 +\ln (1/\delta))} +2}  &\leq
\sqrt{\frac{T}{\alpha n}} \sqrt{2\sqrt{2n(K \ln2 +\ln (1/\delta))} }
\\ &= \left(\frac{8T^2(K \ln2 +\ln (1/\delta))}{\alpha^2 n}
\right)^{1/4},\end{split}\end{equation*} and the theorem follows.
\end{proof}

\section{Pseudo Robustness}\label{sec.pseudo}
In this section we propose a relaxed definition of robustness that
accounts for the case where Equation~(\ref{equ.robust}) holds for
most of training samples, as opposed to Definition~\ref{def.robust}
where Equation~(\ref{equ.robust}) holds for all training samples.
Recall that the size of training set is fixed as $n$.
\begin{definition}Algorithm $\mathcal{A}$ is {\em $(K,\, \epsilon(\mathbf{s}),\hat{n})$ pseudo robust} if
$\mathcal{Z}$ can be partitioned into $K$ disjoint sets, denoted as
$\{C_i\}_{i=1}^K$, and a subset of training samples
$\hat{\mathbf{s}}$ with $|\hat{\mathbf{s}}|=\hat{n}$ such that
$\forall s\in \hat{\mathbf{s}}$,
\[s, z\in C_i,\quad \Longrightarrow \quad \left|l(\mathcal{A}_\mathbf{s}, s)-l(\mathcal{A}_\mathbf{s}, z)\right| \leq \epsilon(\mathbf{s}).\]
\end{definition}
Observe that  $(K,\, \epsilon(\mathbf{s}))$-robust is equivalent to
$(K,\, \epsilon(\mathbf{s}),n)$ pseudo robust.

\begin{theorem}\label{thm.mainpseudo}If $\mathbf{s}$ consists of $n$ i.i.d. samples, and  $\mathcal{A}$ is $(K, \epsilon(\mathbf{s}), \hat{n})$ pseudo robust,
then for any $\delta>0$, with probability at least $1-\delta$,
\[\left|\hat{l}(\mathcal{A}_\mathbf{s})-l_{\mathrm{emp}}(\mathcal{A}_{\mathbf{s}})\right|\leq
\frac{\hat{n}}{n}\epsilon(s)+M\left(\frac{n-\hat{n}}{n}+\sqrt{\frac{2K\ln
2 + 2\ln(1/\delta)}{n}}\right).
\]
\end{theorem}
\begin{proof}
 Let $N_i$ and $\hat{N}_i$ be the set of indices of points of
 $\mathbf{s}$ and $\hat{\mathbf{s}}$
 that fall into the $C_i$, respectively.
  Similarly to the proof of Theorem~\ref{thm.main}, we note that  $(|N_1|, \cdots,
|N_K|)$  is an IID multinomial random variable with parameters $n$
and $(\mu(C_1),\cdots, \mu(C_K))$. And hence due to
Breteganolle-Huber-Carol inequality, the following holds with
probability at least $1-\delta$,
\begin{equation}\label{equ.proofmainpseudo}\sum_{i=1}^{K}\left|\frac{|N_i|}{n} -\mu(C_i)\right| \leq
\sqrt{\frac{2K\ln 2 + 2\ln(1/\delta)}{n}}.\end{equation}
Furthermore, we have
\begin{equation*}\begin{split}&\left|\hat{l}(\mathcal{A}_\mathbf{s})-l_{\mathrm{emp}}
(\mathcal{A}_{\mathbf{s}})\right|
\\= &\left|\sum_{i=1}^K \mathbb{E} \big(l(\mathcal{A}_\mathbf{s},
z)|z\in C_i\big)\mu(C_i) -\frac{1}{n}\sum_{i=1}^n  l(\mathcal{A}_{\mathbf{s}}, s_i) \right|\\
{\leq}  &\left|\sum_{i=1}^K \mathbb{E}
\big(l(\mathcal{A}_\mathbf{s}, z)|z\in
C_i\big)\frac{|N_i|}{n}-\frac{1}{n}\sum_{i=1}^n
l(\mathcal{A}_{\mathbf{s}}, s_i) \right|\\&\qquad+\left|\sum_{i=1}^K
\mathbb{E} \big(l(\mathcal{A}_\mathbf{s}, z)|z\in C_i\big)\mu(C_i)
-\sum_{i=1}^K \mathbb{E} \big(l(\mathcal{A}_\mathbf{s}, z)|z\in
C_i\big)\frac{|N_i|}{n}\right|\\
\leq &\left|\frac{1}{n}\sum_{i=1}^K \big[|N_i|\times\mathbb{E}
\big(l(\mathcal{A}_\mathbf{s}, z)|z\in C_i\big)-\sum_{j\in
\hat{N}_i} l(\mathcal{A}_{\mathbf{s}},s_j)-\sum_{j\in N_i,
j\not\in\hat{N}_i} l(\mathcal{A}_{\mathbf{s}},s_j) \big]
\right|\\&\qquad+\left|\max_{z\in \mathcal{Z}}
|l(\mathcal{A}_{\mathbf{s},
z})|\sum_{i=1}^K\Big|\frac{|N_i|}{n}-\mu(C_i)\Big|
\right|.\end{split}\end{equation*} Note that due to the triangle
inequality as well as the assumption that the loss is non-negative
and upper bounded by $M$, the right-hand side can be upper bounded
by
\begin{equation*}\begin{split}
 &\left|\frac{1}{n} \sum_{i=1}^K\sum_{j\in
\hat{N}_i}\max_{z_2\in C_i} |l(\mathcal{A}_\mathbf{s},
s_j)-l(\mathcal{A}_\mathbf{s}, z_2)|\right| +\left|\frac{1}{n}
\sum_{i=1}^K\sum_{j\in N_i, j\not\in\hat{N}_i}\max_{z_2\in C_i}
|l(\mathcal{A}_\mathbf{s}, s_j)-l(\mathcal{A}_\mathbf{s},
z_2)|\right|\\&\qquad + M \sum_{i=1}^{K}\left|\frac{|N_i|}{n}
-\mu(C_i)\right|\\
\leq & \frac{\hat{n}}{n}\epsilon(s)+ \frac{n-\hat{n}}{n} M + M
\sum_{i=1}^{K}\left|\frac{|N_i|}{n} -\mu(C_i)\right|.
\end{split}\end{equation*}
where the inequality holds due to definition of $N_i$ and
$\hat{N}_i$.
 The theorem follows by applying~(\ref{equ.proofmainpseudo}).
\end{proof}

Similarly, Theorem~\ref{thm.quantile} can be generalized to the
pseudo robust case. The proof is lengthy and hence postponed to
Appendix~\ref{app.proofof.pseudoquantile}.
\begin{theorem}[Quantile Value \& Truncated
Expectation]\label{thm.pseudoquantile} Suppose $\mathbf{s}$ has $n$
 samples drawn i.i.d.\ according to $\mu$, and denote the empirical
distribution  of $\mathbf{s}$ as $\mu_{\mathrm{emp}}$. Let
$\lambda_0 =\sqrt{\frac{2K\ln 2 + 2\ln(1/\delta)}{n}}$. Suppose
$0\leq \beta-\lambda_0-(n-\hat{n})/n\leq
\beta+\lambda_0+(n-\hat{n})/n\leq 1$ and $\mathcal{A}$ is $(K,
\epsilon(\mathbf{s}), \hat{n})$ pseudo robust. Then with probability
at least $1-\delta$, the followings hold
\begin{equation*}\begin{split}(I)\quad &\mathcal{Q}\left(\mathcal{A}_{\mathbf{s}},
\beta-\lambda_0-\frac{n-\hat{n}}{n},\mu_{\mathrm{emp}}\right)
-\epsilon(\mathbf{s})\\ &\qquad\qquad \leq
\mathcal{Q}\left(\mathcal{A}_{\mathbf{s}}, \beta,\mu\right)
 \leq
\mathcal{Q}\left(\mathcal{A}_{\mathbf{s}}, \beta+\lambda_0
+\frac{n-\hat{n}}{n},\mu_{\mathrm{emp}}\right)
+\epsilon(\mathbf{s});\\
(II)\quad &\mathcal{T}\left(\mathcal{A}_{\mathbf{s}},
\beta-\lambda_0-\frac{n-\hat{n}}{n},\mu_{\mathrm{emp}}\right)
-\epsilon(\mathbf{s}) \\ &\qquad\qquad \leq
\mathcal{T}\left(\mathcal{A}_{\mathbf{s}}, \beta,\mu\right)
 \leq
\mathcal{T}\left(\mathcal{A}_{\mathbf{s}}, \beta+\lambda_0
+\frac{n-\hat{n}}{n},\mu_{\mathrm{emp}}\right)
+\epsilon(\mathbf{s}).
\end{split}\end{equation*}
\end{theorem}

\section{Examples of Robust Algorithms}\label{sec.example}
In this section we provide some examples of robust algorithms. The
proofs of the examples can be found in Appendix. Our first example
is Majority Voting (MV) classification \citep[cf Section 6.3
of][]{DevroyeGyorfiLugosi96} that partitions the input space
$\mathcal{X}$ and labels each partition set according to a majority
vote of the training samples belonging to it.
\begin{example}[Majority Voting]\label{exm.mv} Let $\mathcal{Y}=\{-1, +1\}$.  Partition $\mathcal{X}$
to $\mathcal{C}_1, \cdots, \mathcal{C}_K$, and use $\mathcal{C}(x)$
to denote the  set to which $x$ belongs. A new sample $x_a\in
\mathcal{X}$ is labeled by
\[\mathcal{A}_\mathbf{s}(x_a)\triangleq \left\{\begin{array}{ll} 1, &
\mbox{if}\,\,\sum_{s_i\in \mathcal{C}(x_a)}
\mathbf{1}(s_{i|y}=1)\geq \sum_{s_i\in \mathcal{C}(x_a)}
\mathbf{1}(s_{i|y}=-1); \\ -1, & \mbox{otherwise.}
\end{array}\right.
\]
If the loss function is $l(\mathcal{A}_s, z) =f(z_{|y},
\mathcal{A}_\mathbf{s}(z_{|x}))$ for some function $f$, then
 MV is $(2K, 0)$ robust.
\end{example}

 MV algorithm has a natural partition of the sample space that
makes it robust.
 Another
  class of robust algorithms
  are
those that have approximately the same testing loss for  testing
samples that are close (in the sense of geometric distance) to each
other, since we can partition the sample space with norm balls. The
next theorem  states that an algorithm is robust if two samples
being close implies that they have similar testing error.
\begin{theorem}\label{thm.nearby}Fix $\gamma>0$ and metric $\rho$ of $\mathcal{Z}$. Suppose $\mathcal{A}$ satisfies
\[\left|l(\mathcal{A}_{\mathbf{s}}, z_1)-l(\mathcal{A}_{\mathbf{s}}, z_2)\right| \leq \epsilon(\mathbf{s}),\quad \forall z_1, z_2: z_1\in \mathbf{s},\, \rho(z_1, z_2)\leq \gamma,\]
and $\mathcal{N}(\gamma/2, \mathcal{Z}, \rho)<\infty$. Then
$\mathcal{A}$ is $\big(\mathcal{N}(\gamma/2, \mathcal{Z},
\rho),\,\epsilon(\mathbf{s}) \big)$-robust.
\end{theorem}
\begin{proof}Let $\{c_1, \cdots, c_{\mathcal{N}(\gamma/2, \mathcal{Z},
\rho)}\}$ be a $\gamma/2$-cover of $\mathcal{Z}$.
 whose existence is
guaranteed by the definition of covering number. Let
$\hat{C}_i=\{z\in \mathcal{Z}| \rho(z, c_i) \leq \gamma/2\}$, and
$C_i=\hat{C}_i \bigcap \big(\bigcup_{j=1}^{i-1} \hat{C}_j\big)^c$.
Thus, $C_1,\cdots, C_{\mathcal{N}(\gamma/2, \mathcal{Z}, \rho)}$ is
a partition of $\mathcal{Z}$, and satisfies
\[z_1, z_2\in C_i \quad\Longrightarrow \quad \rho(z_1, z_2) \leq
\rho(z_1,c_i)+\rho(z_2,c_i)\leq \gamma.\] Therefore,
\[\left|l(\mathcal{A}_{\mathbf{s}}, z_1)-l(\mathcal{A}_{\mathbf{s}}, z_2)\right| \leq
\epsilon(\mathbf{s}),\quad \forall z_1, z_2: z_1\in \mathbf{s},\,
\rho(z_1, z_2)\leq \gamma,\] implies
\[z_1\in \mathbf{s}\,\,z_1, z_2\in C_i \quad\Longrightarrow \left|l(\mathcal{A}_{\mathbf{s}}, z_1)-l(\mathcal{A}_{\mathbf{s}}, z_2)\right| \leq
\epsilon(\mathbf{s}),\] and the theorem follows.
\end{proof}
Theorem~\ref{thm.nearby} immediately leads to the next example:  if
the testing error given the output of an algorithm is Lipschitz
continuous, then the algorithm is robust.
\begin{example}[Lipschitz continuous functions]If $\mathcal{Z}$ is compact w.r.t.\ metric $\rho$,
$l(\mathcal{A}_{\mathbf{s}}, \cdot)$ is Lipschitz continuous with
Lipschitz constant $c(\mathbf{s})$, i.e.,
\[\left|l(\mathcal{A}_{\mathbf{s}}, z_1)-l(\mathcal{A}_{\mathbf{s}},
z_2)\right| \leq c(\mathbf{s}) \rho(z_1,z_2),\quad \forall z_1, z_2
\in \mathcal{Z},
\] then $\mathcal{A}$ is $\big(\mathcal{N}(\gamma/2, \mathcal{Z}, \rho),\,c(\mathbf{s})\gamma\big)$-robust for all $\gamma>0$.
\end{example}

Theorem~\ref{thm.nearby} also implies that SVM, Lasso, feed-forward
neural network and PCA are robust, as stated in
 Example~\ref{exm.svm} to Example~\ref{exm.PCA}. The proofs are
deferred to Appendix~\ref{app.svm} to~\ref{app.PCA}.
\begin{example}[Support Vector Machine]\label{exm.svm}Let
$\mathcal{X}$ be compact.
 Consider the standard SVM
 formulation~\citep{Cortes95,Scholkopf02}
\begin{equation*}\begin{split} \mbox{Minimize:}_{\mathbf{w},d}\quad &c\|w\|_{\mathcal{H}}^2+\frac{1}{n}\sum_{i=1}^n
{\xi_i}\\
\mbox{s. t.} \quad & 1-s_{i|y}[\langle w,\, \phi(s_{i|x}) \rangle+d] \leq \xi_i;\\
&\xi_i\geq 0.
\end{split}\end{equation*}
Here $\phi(\cdot)$ is a feature mapping, $\|\cdot\|_{\mathcal{H}}$
is its RKHS kernel, and $k(\cdot,\cdot)$ is the kernel function.
 Let $l(\cdot, \cdot)$ be the hinge-loss, i.e., $l\big((w,d),
z\big)=[1-z_{|y}(\langle w, \phi(z_{|x})\rangle+d)]^+$, and define
$f_{\mathcal{H}}(\gamma)\triangleq \max_{\mathbf{a},\mathbf{b}\in
\mathcal{X}, \|\mathbf{a}-\mathbf{b}\|_2\leq \gamma }
\big(k(\mathbf{a},
\mathbf{a})+k(\mathbf{b},\mathbf{b})-2k(\mathbf{a},\mathbf{b})\big)$.
If $k(\cdot,\cdot)$ is continuous,
 then for any $\gamma>0$, $f_{\mathcal{H}}(\gamma)$ is finite, and
SVM is $(2\mathcal{N}(\gamma/2, \mathcal{X}, \|\cdot\|_2),
\sqrt{f_{\mathcal{H}}(\gamma)/c})$ robust.
\end{example}

\begin{example}[Lasso]\label{exm.lasso} Let $\mathcal{Z}$ be compact and the loss function be
$l(A_{\mathbf{s}}, z)= |z_{|y}- A_{\mathbf{s}}(z_{|x})|$. Lasso
\citep{Tibshirani96}, which is the following regression formulation:
\begin{equation}\label{equ.lasso}
\begin{split}\min_{w}:\,\, &\frac{1}{n}\sum_{i=1}^n (s_{i|y} -w^\top s_{i|x}
)^2+c\|w\|_1,
\end{split}
\end{equation} is $\big(\mathcal{N}(\gamma/2, \mathcal{Z}, \|\cdot\|_{\infty}),\,(Y(\mathbf{s})/c+1)\gamma\big)$-robust for all
$\gamma>0$, where $Y(\mathbf{s})\triangleq
\frac{1}{n}\sum_{i=1}^n{s_{i|y}}^2$ .
\end{example}

\begin{example}[Feed-forward Neural Networks]\label{exm.NN}Let $\mathcal{Z}$ be compact and
 the loss function be $l(\mathcal{A}_{\mathbf{s}}, z)= |z_{|y}- \mathcal{A}_{\mathbf{s}}(z_{|x})|$.
Consider  the $d$-layer neural network (trained on $\mathbf{s}$),
which is the following predicting rule given an input $x\in
\mathcal{X}$
\begin{equation*}\begin{split}x^0&:= z_{|x}\\
\forall  v=1,\cdots, d-1:\qquad x^v_i&:= \sigma
(\sum_{j=1}^{N_{v-1}} w^{v-1}_{ij} x^{v-1}_j);\quad i=1,\cdots, N_v;
\\\mathcal{A}_{\mathbf{s}}(x)&:= \sigma (\sum_{j=1}^{N_{d-1}} w^{d-1}_j x^{d-1}_j
);\end{split}\end{equation*}
 If there exists  $\alpha, \beta$ such that the $d$-layer neural network satisfying that
$|\sigma(a)-\sigma(b)|\leq \beta |a-b|$, and $\sum_{j=1}^{N_v}
|w^v_{ij}|\leq \alpha$ for all $v, i$, then  it is
$\big(\mathcal{N}(\gamma/2, \mathcal{Z},
\|\cdot\|_{\infty}),\,\alpha^d\beta^d\gamma\big)$-robust, for all
$\gamma>0$.
\end{example}We remark that in Example~\ref{exm.NN}, the number of hidden units in each layer has no
effect on the robustness of the algorithm and consequently the bound
on the testing error. This indeed agrees with \citet{Bartlett98},
where the author showed (using a different approach based on
fat-shattering dimension) that for neural networks, the weight plays
a more important role than the number of hidden units.

The next example considers an unsupervised learning algorithm,
namely the principal component analysis. We show that it is robust
if the sample space is {\em bounded}. Note that, this does not
contradict with the well known fact that the principal component
analysis is sensitive to outliers which are far away from the
origin.
\begin{example}[Principal Component Analysis (PCA)]\label{exm.PCA} Let $\mathcal{Z}\subset
\mathbb{R}^m$, such that $\max_{z\in \mathcal{Z}} \|z\|_2 \leq B$.
If  the loss function  is $l((w_1,\cdots, w_d),z) =
\sum_{k=1}^d(w_k^\top z)^2$, then finding the first $d$ principal
components, which solves the following optimization problem of
$w_1,\cdots, w_d\in \mathbb{R}^m$,
\begin{equation*}\begin{split}\mbox{Maximize:} \quad &  \sum_{i=1}^n \sum_{k=1}^d (w_k^\top s_i)^2\\
\mbox{Subject to:}\quad & \|w_k\|_2 =1,\quad k=1,\cdots, d;\\
& w_i^\top w_j=0, \quad i\not= j.
\end{split}\end{equation*}
is $(\mathcal{N}(\gamma/2, \mathcal{Z}, \|\cdot\|_2), 2d\gamma
B)$-robust.
\end{example}

The last example is large-margin classification, which is a
generalization of Example~\ref{exm.marginmoti}. We need the
following standard definition \citep[e.g.,][]{Bartlett98} of the
distance of a point to a classification rule.
\begin{definition}Fix a metric $\rho$ of $\mathcal{X}$. Given   a classification rule $\Delta$ and $x\in
\mathcal{X}$, the {\em distance} of $x$ to $\Delta$ is
\[\mathcal{D}(x, \Delta)\triangleq \inf\{c\geq 0|\exists x'\in \mathcal{X}:\, \rho(x, x')\leq c, \Delta(x)\not=\Delta(x')\}.\]
\end{definition}

A large margin classifier is a classification rule such that most of
the training samples are ``far away'' from the classification
boundary.

\begin{example}[Large-margin classifier]\label{exm.margin}If there exist $\gamma$ and
$\hat{n}$ such that
\[\sum_{i=1}^n \mathbf{1}\big(\mathcal{D}(s_{i|x}, \mathcal{A}_s)> \gamma \big) \geq \hat{n},\]
then
 algorithm $\mathcal{A}$ is
$(2\mathcal{N}(\gamma/2, \mathcal{X}, \rho), 0, \hat{n})$ pseudo
robust, provided that $\mathcal{N}(\gamma/2, \mathcal{X},
\rho)<\infty$.
\end{example}
Note that if we take $\rho$ to be the Euclidean norm, and let
$\hat{n}=n$, then we recover Example~\ref{exm.marginmoti}.

\section{Necessity of Robustness}\label{sec.equivalence}
Thus far we have considered finite sample generalization bounds of
robust algorithms. We now turn to asymptotic analysis, i.e., we are
given an increasing set of training samples $\mathbf{s}=(s_1,
s_2,\cdots)$ and   tested on an increasing set of testing samples
$\mathbf{t}=(t_1, t_2,\cdots)$. We use $\mathbf{s}(n)$ and
$\mathbf{t}(n)$ to denote the first $n$ elements of   training
samples and testing samples respectively. For succinctness, we let
$\mathcal{L}(\cdot,\cdot)$ to be the average loss given a set of
samples, i.e., for $h\in \mathcal{H}$,
\[\mathcal{L}(h,\,\mathbf{t}(n))\equiv\frac{1}{n}\sum_{i=1}^n
l(h,t_i).\]

We
 show in this section that robustness is  an essential property of successful
learning. In particular, a (weaker) notion of robustness
characterizes generalizability, i.e., a learning algorithm
generalizes if and only if it is weakly robust. To make this
precise, we define the notion of generalizability and weak
robustness first.

\begin{definition}
\begin{enumerate}
\item
A learning algorithm $\mathcal{A}$ {\em generalizes w.r.t.\
$\mathbf{s}$} if
\[\limsup_{n}\Big\{\mathbb{E}_{t}\Big(l(\mathcal{A}_{\mathbf{s}(n)},
t)\Big)-\mathcal{L}(\mathcal{A}_{\mathbf{s}(n)},\mathbf{s}(n))\Big\}\leq
0.\]
\item A learning algorithm $\mathcal{A}$ {\em generalize w.p.\ 1} if
it generalize  w.r.t.\ almost every $\mathbf{s}$.
\end{enumerate}
\end{definition}
We remark that the proposed notion of generalizability differs
slightly from the standard one in the sense that the latter requires
that the empirical risk and the expected risk converges in mean,
while the proposed notion requires convergence w.p.1. It is
straightforward that the proposed notion implies the standard one.

\begin{definition}
\begin{enumerate}
\item
 A learning algorithm $\mathcal{A}$ is {\em weakly robust w.r.t
$\mathbf{s}$ } if there exists a sequence of
$\{\mathcal{D}_n\subseteq \mathcal{Z}^n\}$ such that
$\pr(\mathbf{t}(n)\in \mathcal{D}_n) \rightarrow 1$, and
\[\limsup_{n}\left\{\max_{\hat{\mathbf{s}}(n)\in \mathcal{D}_n}\big[ \mathcal{L}(\mathcal{A}_{\mathbf{s}(n)}, \hat{\mathbf{s}}(n))- \mathcal{L}(\mathcal{A}_{\mathbf{s}(n)},  \mathbf{s}(n)) \big]\right\}\leq
0.\] \item A learning algorithm $\mathcal{A}$ is {\em a.s. weakly
robust} if it is robust  w.r.t. almost every $\mathbf{s}$.
\end{enumerate}
\end{definition}
We briefly comment on the definition of weak robustness. Recall that
the definition of robustness requires that the sample space can be
partitioned into disjoint subsets such that if a testing sample
belongs to the same partitioning set of a training sample, then they
have similar loss. Weak robustness generalizes such notion by
considering the average loss of testing samples and training
samples. That is, if for a large (in the probabilistic sense) subset
of $\mathcal{Z}^n$, the testing error is close to the training
error, then the algorithm is weakly robust. It is easy to see, by
Breteganolle-Huber-Carol lemma, that if for any fixed $\epsilon>0$
there exists $K$ such that $\mathcal{A}$ is $(K, \epsilon)$ robust,
then $\mathcal{A}$ is weakly robust.

We now establish the main result of this section: weak robustness
and generalizability are equivalent.

\begin{theorem}\label{thm.ns} An algorithm $\mathcal{A}$ generalizes
w.r.t.\
$\mathbf{s}$ if and only if it is weakly robust w.r.t. $\mathbf{s}$.
\end{theorem}
\begin{proof}
We prove the sufficiency of weak robustness first. When
$\mathcal{A}$ is weakly robust w.r.t. $\mathbf{s}$, by definition
there exists $\{D_n\}$ such that
 for any $\delta,\,\epsilon>0$, there exists $N(\delta,
\epsilon)$ such that for all $n>N(\delta, \epsilon)$,
$\mathrm{Pr}(\mathbf{t}(n)\in D_n)> 1-\delta$, and
\begin{equation}\label{equ.proofinns}\sup_{\hat{\mathbf{s}}(n)\in D_n}
\mathcal{L}(\mathcal{A}_{\mathbf{s}(n)},\hat{\mathbf{s}}(n))-\mathcal{L}(\mathcal{A}_{\mathbf{s}(n)},
\mathbf{s}(n))< \epsilon.\end{equation} Therefore, the following
holds for any $n>N(\delta, \epsilon)$,
\begin{equation*}
\begin{split}&\mathbb{E}_{t}\Big(l(\mathcal{A}_{\mathbf{s}(n)},
t)\Big)-\mathcal{L}(\mathcal{A}_{\mathbf{s}(n)},\mathbf{s}(n))
\\=&\mathbb{E}_{\mathbf{t}(n)}\Big(\mathcal{L}(\mathcal{A}_{\mathbf{s}(n)},
\mathbf{t}(n))\Big)-\mathcal{L}(\mathcal{A}_{\mathbf{s}(n)},\mathbf{s}(n))\\
= &\mathrm{Pr}(\mathbf{t}(n)\not\in
D_n)\mathbb{E}\Big(\mathcal{L}(\mathcal{A}_{\mathbf{s}(n)},\,
\mathbf{t}(n))|\mathbf{t}(n)\not\in D_n
\Big)+\mathrm{Pr}(\mathbf{t}(n)\in
D_n)\mathbb{E}\Big(\mathcal{L}(\mathcal{A}_{\mathbf{s}(n)},
\mathbf{t}(n))|\mathbf{t}(n)\in D_n \Big)
\\&\quad-\mathcal{L}(\mathcal{A}_{\mathbf{s}(n)},\mathbf{s}(n))\\
{\leq}& \delta M+\sup_{\hat{\mathbf{s}}(n)\in D_n}\big\{
\mathcal{L}(\mathcal{A}_{\mathbf{s}(n)},\hat{\mathbf{s}}(n))
-\mathcal{L}(\mathcal{A}_{\mathbf{s}(n)},\mathbf{s}(n)) \big\}\leq
\delta M+\epsilon.
\end{split}
\end{equation*}
Here, the first equality holds by i.i.d. of $\mathbf{t}(n)$, and the
second equality holds by conditional expectation. The inequalities
hold due to the assumption that the loss function is upper bounded
by $M$, as well as~(\ref{equ.proofinns}).

 We thus conclude
that
 the algorithm $\mathcal{A}$ generalizes for
$\mathbf{s}$, because $\epsilon,\,\delta$ can be arbitrary.

Now we turn to the necessity of weak robustness. First, we establish
the following lemma.
\begin{lemma}\label{pro.nonrobust}Given $\mathbf{s}$, if algorithm $\mathcal{A}$
is not weakly robust w.r.t.~$\mathbf{s}$,  then there exists
$\epsilon^*,\,\delta^*>0$ such that the following holds for
infinitely many $n$,
\begin{equation}\label{equ.generalnonrobust}\mathrm{Pr}\Big(\mathcal{L}(\mathcal{A}_{\mathbf{s}(n)},\mathbf{t}(n))\geq
\mathcal{L}(\mathcal{A}_{\mathbf{s}(n)},\mathbf{s}(n))+\epsilon^*\Big)\geq
\delta^*.\end{equation}
\end{lemma}
\begin{proof}
We prove the lemma by contradiction. Assume that such $\epsilon^*$
and $\delta^*$ do not exist. Let $\epsilon_v=\delta_v=1/v$ for
$v=1,2\cdots$, then there exists a non-decreasing sequence
$\{N(v)\}_{v=1}^\infty$ such that for all $v$,  if $n\geq N(v)$ then
$\mathrm{Pr}\Big(\mathcal{L}(\mathcal{A}_{\mathbf{s}(n)},\mathbf{t}(n))\geq
\mathcal{L}(\mathcal{A}_{\mathbf{s}(n)},\mathbf{s}(n))+\epsilon_v\Big)<
\delta_v$. For each $n$, define the following set:
\[\mathcal{D}^v_n\triangleq \Large\{\hat{\mathbf{s}}(n)|\mathcal{L}(\mathcal{A}_{\mathbf{s}(n)},\hat{\mathbf{s}}(n))-
\mathcal{L}(\mathcal{A}_{\mathbf{s}(n)},\mathbf{s}(n))
<\epsilon_v\Large\}.\] Thus, for $n\geq N(v)$ we have
\begin{equation*}\begin{split}
\mathrm{Pr}(\mathbf{t}(n)\in
\mathcal{D}_n^v)&=1-\mathrm{Pr}\Big(\mathcal{L}(\mathcal{A}_{\mathbf{s}(n)},\mathbf{t}(n))\geq
\mathcal{L}(\mathcal{A}_{\mathbf{s}(n)},\mathbf{s}(n))+\epsilon_v\Big)>
1-\delta_v.\end{split}\end{equation*} For $n\geq N(1)$, define
$\mathcal{D}_n\triangleq \mathcal{D}_n^{v(n)}$, where:
$v(n)\triangleq\max\big(v|N(t)\leq n;\,\,v\leq n\big)$. Thus for all
$n\geq N(1)$ we have that $\mathrm{Pr}(\mathbf{t}(n)\in
\mathcal{D}_n)> 1-\delta_{v(n)}$ and $\sup_{\hat{\mathbf{s}}(n)\in
\mathcal{D}_n}
\mathcal{L}(\mathcal{A}_{\mathbf{s}(n)},\hat{\mathbf{s}}(n))-\mathcal{L}(\mathcal{A}_{\mathbf{s}(n)},\mathbf{s}(n))<
\epsilon_{v(n)}$. Note that  $v(n)\uparrow \infty$, it follows that
$\delta_{v(n)}\rightarrow 0$ and $\epsilon_{v(n)}\rightarrow 0$.
Therefore, $\mathrm{Pr}(\mathbf{t}(n)\in \mathcal{D}_n)\rightarrow
1$, and
\[\limsup_{n\rightarrow\infty}\Big\{\sup_{\hat{\mathbf{s}}(n)\in \mathcal{D}_n}
\mathcal{L}(\mathcal{A}_{\mathbf{s}(n)},\hat{\mathbf{s}}(n))-\mathcal{L}(\mathcal{A}_{\mathbf{s}(n)},\mathbf{s}(n))\Big\}\leq
0.\] That is, $\mathcal{A}$ is weakly robust w.r.t. $\mathbf{s}$,
which is a desired contradiction.
\end{proof}

We now prove the necessity of weak robustness. Recall that
$l(\cdot,\cdot)$ is uniformly bounded. Thus by Hoeffding's
inequality we have that for any $\epsilon,\delta$, there exists
$n^*$ such that for any $n>n^*$, with probability at least
$1-\delta$, we have
 $\Big|\frac{1}{n}\sum_{i=1}^n
l(\mathcal{A}_{\mathbf{s}(n)},\,t_i) -\mathbb{E}_{t}
 (l(\mathcal{A}_{\mathbf{s}(n)},\,t))\Big|\leq {\epsilon}$.
This implies that
\begin{equation}\label{equ.convergeinprob}\mathcal{L}(\mathcal{A}_{\mathbf{s}(n)},
\mathbf{t}(n))-\mathbb{E}_{t}l(\mathcal{A}_{\mathbf{s}(n)},
t)\,\,\stackrel{\mathrm{Pr}}{\longrightarrow} 0.\end{equation}
 Since algorithm
$\mathbb{A}$ is not robust, Lemma~\ref{pro.nonrobust} implies that
(\ref{equ.generalnonrobust}) holds for infinitely many $n$. This,
combined with Equation~(\ref{equ.convergeinprob}) implies that  for
infinitely many $n$,
\[\mathbb{E}_{t}l(\mathcal{A}_{\mathbf{s}(n)},
t)\geq
\mathcal{L}(\mathcal{A}_{\mathbf{s}(n)},\mathbf{s}(n))+\frac{\epsilon^*}{2},\]
which means that $\mathcal{A}$ does not generalize. Thus, the
necessity of weak robustness is established.
\end{proof}

Theorem~\ref{thm.ns} immediately leads to the following corollary.
\begin{corollary}An algorithm $\mathcal{A}$ generalizes w.p.\ 1 if
and only if it is a.s.\ weakly robust.
\end{corollary}

\section{Discussion}
In this paper we investigated the generalization ability of learning
algorithm based on their robustness:  the property that if a testing
sample is ``similar'' to a training sample, then its loss is close
to the training error. This provides a novel approach, different
from the complexity or stability argument, in studying the
performance of learning algorithms. We further showed that a weak
notion of robustness characterizes generalizability, which implies
that robustness is a fundamental property for learning algorithms to
work.

Before concluding the paper, we outline several directions for
future research.
\begin{itemize}
\item {\em Adaptive partition:} In Definition~\ref{def.robustalgorithm} when the notion of robustness
was introduced, we required that the partitioning of $\mathcal{Z}$
into $K$ sets is {\em fixed}. That is, regardless of the training
sample set, we  partition $\mathcal{Z}$ into the same $K$ sets. A
  natural and interesting question is what if such fixed
partition does not exist, while instead we can only partition
$\mathcal{Z}$ into $K$ sets {\em adaptively}, i.e., for different
training set we will have a different partitioning of $\mathcal{Z}$.
Adaptive partition setup can be used to study algorithms such as
k-NN. Our current proof technique does not straightforwardly extend
to such a setup, and we would like to understand whether a
meaningful generalization bound under this weaker notion of
robustness can be obtained.
\item {\em Mismatched datasets:} One advantage of
algorithmic robustness framework is the ability to handle
non-standard learning setups. For example, in
Section~\ref{sss.quantile} and~\ref{sec.Markovian} we derived
generalization bounds for quantile loss and for samples drawn from a
Markovian chain, respectively. A problem of the same essence is the
{\em mismatched datasets}, where the training samples are generated
according to a distribution slightly different  from that of the
testing samples, e.g., the two distributions may have a small K-L
divergence. We conjecture that in this case a generalization bound
similar to Theorem~\ref{thm.main} would be possible, with  an extra
term depending on the magnitude of the difference of the two
distributions.
\item {\em Outlier removal:} One possible reason that the training
samples is generated differently from the testing sample is outlier
corruption. It is often the case that  the training sample set is
corrupted by some outliers. In addition, algorithms designed to be
outlier resistent abound in the
literature~\citep[e.g.,][]{Huber81,RousseeuwLeroy87}. The robust
framework may provide a novel approach in studying both the
generalization ability and the outlier resistent property of these
algorithms. In particular, the results reported in
Section~\ref{sss.quantile} can serve as a starting point of future
research in this direction.
\item {\em Consistency:} We addressed in this paper the relationship
between robustness and generalizability. An equally important
feature of learning algorithms is {\em consistency}: the property
that  a learning algorithm guarantees to recover the global optimal
solution as the number of training data increases. While it is
straightforward that if an algorithm minimizes the empirical error
asymptotically and also generalizes (or equivalently is {\em weakly
robust}), then it is consistent, much less is known for a necessary
condition for an algorithm to be consistent. It is certainly
interesting to investigate the relationship between consistency and
robustness, and in particular whether robustness is necessary for
consistency, at least for algorithms that asymptotically minimize
the empirical error.
\item {\em Other robust algorithms:} The
proposed robust approach considers a general learning setup.
However, except for PCA, the algorithms investigated in
Section~\ref{sec.example} all belong to the supervised learning
setting. One natural extension is to investigate other robust
unsupervised and semi-supervised learning algorithms. One difficulty
is that compared to supervised learning case, the analysis of
unsupervised/semi-supervised learning algorithms can be challenging,
due to the fact that many of them are random iterative algorithms
(e.g., k-means).

\end{itemize}

\appendix
\section{Proofs}
\subsection{Proof of Theorem~\ref{thm.pseudoquantile}}\label{app.proofof.pseudoquantile}
We  observe the following properties of quantile value and truncated
mean:
\begin{enumerate}
\item If $X$ is supported on $\mathbb{R}^+$ and $\beta_1\geq \beta_2$, then
\[\mathbb{Q}^{\beta_1}(X) \geq \mathbb{Q}^{\beta_2}(X); \quad \mathbb{T}^{\beta_1}(X) \geq \mathbb{T}^{\beta_2}(X).\]
\item If $Y$ stochastically dominates $X$, i.e., $\mathrm{Pr}(Y \geq a) \geq \mathrm{Pr}(X \geq
a)$ for all $a\in \mathbb{R}$, then for any $\beta$,
\[\mathbb{Q}^\beta(Y) \geq \mathbb{Q}^\beta(X);\quad \mathbb{T}^\beta(Y) \geq \mathbb{T}^\beta(X).\]
\item The $\beta$-truncated mean of empirical distribution of
nonnegative  $(x_1, \cdots, x_n)$ is given by
\[\min_{\alpha: 0\leq \alpha_i\leq 1/n,\, \sum_{i=1}^n \alpha_i\leq \beta} \sum_{i=1}^n \alpha_i x_i.\]
\end{enumerate}

By definition of pseudo-robustness, $\mathcal{Z}$ can be partitioned
into $K$ disjoint sets, denoted as $\{C_i\}_{i=1}^K$, and a subset
of training samples $\hat{\mathbf{s}}$ with
$|\hat{\mathbf{s}}|=\hat{n}$ such that
\[z_1\in \hat{\mathbf{s}},\,\,  z_1, z_2 \in C_i,\quad \Longrightarrow \quad \left|l(\mathcal{A}_\mathbf{s}, z_1)-l(\mathcal{A}_\mathbf{s}, z_2)\right| \leq \epsilon(\mathbf{s});\quad \forall s.\]

 Let $N_i$  be the set of index of points of
 $\mathbf{s}$
 that fall into the $C_i$.
   Let
$\mathcal{E}$ be the event that the following holds:
\begin{equation*}\sum_{i=1}^{K}\left|\frac{|N_i|}{n} -\mu(C_i)\right| \leq
\sqrt{\frac{2K\ln 2 + 2\ln(1/\delta)}{n}}.\end{equation*}From the
proof of Theorem~\ref{thm.main}, $\mathrm{Pr}(\mathcal{E}) \geq
1-\delta$. Hereafter we restrict the discussion to the case when
$\mathcal{E}$ holds.

Denote
\[v_j=\arg\min_{z\in \mathcal{C}_j} l(\mathcal{A}_{\mathbf{s}},
z).\] By symmetry, without loss of generality we assume that $0\leq
l(\mathcal{A}_{\mathbf{s}}, v_1)\leq l(\mathcal{A}_{\mathbf{s}},
v_2)\leq\cdots\leq l(\mathcal{A}_{\mathbf{s}}, v_K) \leq M$.
 Define a set of
samples $\tilde{\mathbf{s}}$ as
\[\tilde{s}_i=\left\{\begin{array}{ll} s_i & \mbox{if } s_i\in \hat{\mathbf{s}} ;\\ v_j &\mbox{if }s_i\not\in \hat{\mathbf{s}} ,\,s_i\in \mathcal{C}_j.\end{array}\right.\]
Define  discrete probability measures $\hat{\mu}$ and $\tilde{\mu}$,
supported on $\{v_1, \cdots, v_K\}$ as
\[\hat{\mu}(\{v_j\}) =\mu(\mathcal{C}_j);\quad \tilde{\mu}(\{v_j\})=
\frac{|N_j|}{n}.\] Further,  let $\tilde{\mu}_{\mathrm{emp}}$ denote
the empirical distribution of sample set $\tilde{\mathbf{s}}$.

{\bf Proof of (I)}:

 Observe that $\mu$ stochastically dominates $\hat{\mu}$, hence
\begin{equation}\label{equ.proofinquantileobserve}\mathcal{Q}(\mathcal{A}_{\mathbf{s}},\beta, \hat{\mu}) \leq
\mathcal{Q}(\mathcal{A}_{\mathbf{s}},\beta, \mu).\end{equation} Also
by definition of $\mathcal{Q}(\cdot)$ and $\hat{\mu}$,
\[\mathcal{Q}(\mathcal{A}_{\mathbf{s}},\beta, \hat{\mu})=
v_{k^*};\quad \mbox{where:}\,\, k^*=\min\{k: \sum_{i=1}^k
\hat{\mu}(v_i) \geq \beta\}.\] Let $\overline{\mathbf{s}}$ be the
set of all samples $s_i$ such that $s_i\in \hat{\mathbf{s}}$, and
$s_i \in \mathcal{C}_j$ for some $j\leq k^*$. Observe that
\begin{equation}\label{equ.proofofquatile}\forall s_i \in \hat{\mathbf{s}}:\,\,
l(\mathcal{A}_{\mathbf{s}}, s_i) \leq
v_{k^*}+\epsilon(\mathbf{s})=\mathcal{Q}(\mathcal{A}_{\mathbf{s}},\beta,
\hat{\mu})+\epsilon(\mathbf{s}) .\end{equation}

Note that $\mathcal{E}$ implies
\[\frac{1}{n} \sum_{j=1}^{k^*} \sum_{s_i\in \mathcal{C}_j} 1 \geq \sum_{j=1}^{k^*} \mu(\mathcal{C}_j)-\lambda_0 =\sum_{j=1}^k
\hat{\mu}(v_j)-\lambda_0\geq \beta-\lambda_0.\] Since
$\mathcal{A}_{\mathbf{s}}$ is pseudo robust, we have
\[\frac{1}{n}\sum_{s_i\not\in \hat{\mathbf{s}}}
=\frac{n-\hat{n}}{n}.\] Therefore
\[\frac{1}{n}\sum_{j=1}^{k^*}\sum_{s_i\in \overline{\mathbf{s}}, s_i\in \mathcal{C}_j} 1 \geq \frac{1}{n}\sum_{j=1}^{k^*}\sum_{s_i\in \mathcal{C}_j} 1 - \frac{1}{n}\sum_{s_i \not\in\hat{\mathbf{s}}} 1\geq \beta-\lambda_0-\frac{n-\hat{n}}{n}.\]
Thus, $\overline{\mathbf{s}}$ is a subset of $\mathbf{s}$ of at
least $n(\beta-\lambda_0-(n-\hat{n})/n)$ elements. Thus
(\ref{equ.proofinquantileobserve}) and (\ref{equ.proofofquatile})
lead to
\[\mathcal{Q}(\mathcal{A}_{\mathbf{s}}, \beta-\lambda_0-(n-\hat{n})/n, \mu_{\mathrm{emp}}) \leq \max \{s_i:\,s_i\in \overline{\mathbf{s}}\}\leq \mathcal{Q}(\mathcal{A}_{\mathbf{s}},\beta, \mu)+\epsilon(\mathbf{s}).\]
Thus, we establish the left inequality. The proof of the right one
is identical and hence omitted.

{\bf Proof of (II):}

 The proof constitutes four steps.

{\bf Step 1:} Observe that  $\mu$ stochastically dominates
$\hat{\mu}$, hence
\[\mathcal{T}(\mathcal{A}_{\mathbf{s}},\beta, \hat{\mu}) \leq \mathcal{T}(\mathcal{A}_{\mathbf{s}},\beta, \mu).\]

{\bf Step 2:} We prove that
\[\mathcal{T}(\mathcal{A}_{\mathbf{s}},\beta-\lambda_0, \tilde{\mu}) \leq \mathcal{T}(\mathcal{A}_{\mathbf{s}},\beta, \hat{\mu}).\]
Note that t  $\mathcal{E}$ implies  for all $j$, we have
\[\tilde{\mu}(\{v_1,\cdots, v_j\})-\lambda_0 \leq \hat{\mu}
(\{v_1,\cdots, v_j\}),\]

Therefore, there uniquely exists a non-negative integer $j^*$ and a
$c^* \in [0,1)$ such that
\[\hat{\mu}(\{v_1, \cdots, v_{j^*}\})+ c^*\hat{\mu}(\{v_{j^*+1}\})
=\beta,\]and define
\begin{equation}\label{equ.proofofquantile}\hat{\beta}=\sum_{i=1}^{j^*} \min
(\tilde{\mu}(\{v_i\}), \hat{\mu}(\{v_i\})) + c^*\min
(\tilde{\mu}(\{v_{j^*+1}\}),
\hat{\mu}(\{v_{j^*+1}\})),\end{equation} then we have
$\hat{\beta}\geq \beta-\lambda_0$, which leads to
\begin{equation*}\begin{split}&\mathcal{T}(\mathcal{A}_{\mathbf{s}},\beta-\lambda_0, \tilde{\mu})
 \leq \mathcal{T}(\mathcal{A}_{\mathbf{s}},\hat{\beta},
\tilde{\mu})\\
\stackrel{(a)}{\leq} &\sum_{i=1}^{j^*}
l(\mathcal{A}_{\mathbf{s}},v_i) \min (\tilde{\mu}(\{v_i\}),
\hat{\mu}(\{v_i\})) + c^* l(\mathcal{A}_{\mathbf{s}},v_{j^*+1})\min
(\tilde{\mu}(\{v_{j^*+1}\}), \hat{\mu}(\{v_{j^*+1}\}))\\
\leq &\sum_{i=1}^{j^*} l(\mathcal{A}_{\mathbf{s}},v_i)
\hat{\mu}(\{v_i\}) + c^* l(\mathcal{A}_{\mathbf{s}},v_{j^*+1})
\hat{\mu}(\{v_{j^*+1}\})
=\mathcal{T}(\mathcal{A}_{\mathbf{s}},\beta, \hat{\mu}),
\end{split}\end{equation*}where $(a)$ holds because Equation~(\ref{equ.proofofquantile}) essentially means that
$\mathcal{T}(\mathcal{A}_{\mathbf{s}},\hat{\beta}, \tilde{\mu})$ is
a weighted sum with total weights equals to $\hat{\beta}$, which
puts more weights on small terms, and hence is smaller.

{\bf Step 3:} We prove that \[
\mathcal{T}(\mathcal{A}_{\mathbf{s}},\beta-\lambda_0,
\tilde{\mu}_{\mathrm{emp}})-\epsilon(\mathbf{s}) \leq
\mathcal{T}(\mathcal{A}_{\mathbf{s}},\beta-\lambda_0,
\tilde{\mu}).\] Let $\tilde{\mathbf{t}}$ be a set of $n$ samples,
such that $N_j$ of them are $v_j$ for $j=1,\cdots, K$. Observe that
 $\tilde{\mu}$  is the empirical distribution of
$\tilde{\mathbf{t}}$. Further note that there is a one-to-one
mapping between samples in $\tilde{\mathbf{s}}$ and that in
$\tilde{\mathbf{t}}$ such that each pair (say $\tilde{s}_i,
\tilde{t}_i$) of samples belongs to the same $\mathcal{C}_j$. By
definition of $\tilde{\mathbf{s}}$ this guarantees that
$|l(\mathcal{A}_{\mathbf{s}},\tilde{s}_i
)-l(\mathcal{A}_{\mathbf{s}},\tilde{t}_i )| \leq
\epsilon(\mathbf{s})$, which implies
\[
\mathcal{T}(\mathcal{A}_{\mathbf{s}},\beta-\lambda_0,
\tilde{\mu}_{\mathrm{emp}})-\epsilon(\mathbf{s}) \leq
\mathcal{T}(\mathcal{A}_{\mathbf{s}},\beta-\lambda_0,
\tilde{\mu}).\]

{\bf Step 4:} We prove that \[
\mathcal{T}(\mathcal{A}_{\mathbf{s}},\beta-\lambda_0-\frac{n-\hat{n}}{n},
\mu_{\mathrm{emp}}) \leq
\mathcal{T}(\mathcal{A}_{\mathbf{s}},\beta-\lambda_0,
\tilde{\mu}_{\mathrm{emp}}). \]

Let $\mathbb{I} =\{i: s_i=\tilde{s}_i\}$), the following holds:
\[\sum_{i=1}^n \alpha_i l(\mathcal{A}_{\mathbf{s}}, \tilde{s}_i) \geq \sum_{i\in \mathbb{I}} \alpha_i l(\mathcal{A}_{\mathbf{s}}, \tilde{s}_i) = \sum_{i\in \mathbb{I}} \alpha_i l(\mathcal{A}_{\mathbf{s}}, s_i);\quad \forall \alpha: 0\leq \alpha_i\leq \frac{1}{n};\,\, \sum_{i=1}^n{\alpha_i}=\beta-\lambda_0.\]
Note that $|\{i\not \in\mathbb{I}\}|=n-\hat{n}$, then $\sum_{i\in
\mathbb{I}} \alpha_i \geq \beta-\lambda_0 -\frac{n-\hat{n}}{n}$.
Thus we have $\forall \alpha: 0\leq \alpha_i\leq \frac{1}{n};\,\,
\sum_{i=1}^n{\alpha_i}=\beta-\lambda_0$,
\[\sum_{i\in \mathbb{I}} \alpha_i l(\mathcal{A}_{\mathbf{s}}, s_i) \geq \min_{\alpha': 0\leq \alpha_i' \leq \frac{1}{n}, \, \sum_{i=1}^n {\alpha_i'} \leq \beta-\lambda_0 -\frac{n-\hat{n}}{n}} \sum_{i=1}^n \alpha_i' l(\mathcal{A}_{\mathbf{s}}, s_i)=\mathcal{T}(\mathcal{A}_{\mathbf{s}},\beta-\lambda_0,
\tilde{\mu}_{\mathrm{emp}}).\] Therefore,
\[\sum_{i=1}^n \alpha_i l(\mathcal{A}_{\mathbf{s}}, \tilde{s}_i) \geq \mathcal{T}(\mathcal{A}_{\mathbf{s}},\beta-\lambda_0-\frac{n-\hat{n}}{n},
\mu_{\mathrm{emp}});\quad\forall \alpha: 0\leq \alpha_i\leq
\frac{1}{n};\,\, \sum_{i=1}^n{\alpha_i}=\beta-\lambda_0.\]
Minimization over $\alpha$ on both side. We proved
\[
\mathcal{T}(\mathcal{A}_{\mathbf{s}},\beta-\lambda_0-\frac{n-\hat{n}}{n},
\mu_{\mathrm{emp}}) \leq
\mathcal{T}(\mathcal{A}_{\mathbf{s}},\beta-\lambda_0,
\tilde{\mu}_{\mathrm{emp}}). \]

Combining all four steps, we proved the left inequality, i.e.,
\[\mathcal{T}(\mathcal{A}_{\mathbf{s}},\beta-\lambda_0-\frac{n-\hat{n}}{n},
\mu_{\mathrm{emp}}) -\epsilon(\mathbf{s}) \leq \mathcal{T}(
\mathcal{A}_{\mathbf{s}},\beta, \mu).\] The right inequality can be
proved  identically and hence omitted.

\subsection{Proof of Example~\ref{exm.mv}}\label{app.mv}
We can partition $\mathcal{Z}$ as $\{-1\}\times
\mathcal{C}_1,\cdots, \{-1\}\times \mathcal{C}_K, \{+1\}\times
\mathcal{C}_1,\cdots, \{+1\}\times \mathcal{C}_K$. Consider $z_a,
z_b$ that belong to a same set, then $z_{a|y}=z_{b|y}$, and $\exists
i$ such that $z_{a|x},z_{b|x} \in \mathcal{C}_i$, which by the
definition of Majority Voting algorithm implies that
$\mathcal{A}_{\mathbf{s}}(z_{a|x})=\mathcal{A}_{\mathbf{s}}(z_{b|x})$.
Thus, we have
\[l(\mathcal{A}_{\mathbf{s}}, z_a)=f(z_{a|y}, \mathcal{A}_{\mathbf{s}}(z_{a|x}))
=f(z_{b|y},
\mathcal{A}_{\mathbf{s}}(z_{b|x}))=l(\mathcal{A}_{\mathbf{s}},
z_b).\] Hence MV is $(2K, 0)$-robust.

\subsection{Proof of Example~\ref{exm.svm}}\label{app.svm}
The existence of $f_{\mathcal{H}}(\gamma)$ follows from the
compactness of $\mathcal{X}$ and continuity of $k(\cdot,\cdot)$.

To prove the robustness of SVM, let $(w^*, d^*)$ be the solution
given training data $\mathbf{s}$.
 To  avoid notation clutter, let $y_i=s_{i|y}$ and $x_i=s_{i|x}$.
Thus, we have (due to optimality of $w^*, d^*$)
\[c\|w^*\|_\mathcal{H}^2+\frac{1}{n}\sum_{i=1}^n[1-y_i(\langle w^*,\, \phi(x_i)\rangle+d^*)]^+ \leq c\|0\|_\mathcal{H}^2+\frac{1}{n}\sum_{i=1}^n[1-y_i(\langle 0,\, \phi(x_i)\rangle+0)]^+ =1,\]
which implies $\|w^*\|_{\mathcal{H}}\leq \sqrt{1/c}$. Let $c_1,
\cdots, c_{\mathcal{N}(\gamma/2, \mathcal{X}, \|\cdot\|_2)}$ be a
$\gamma/2$-cover of $\mathcal{X}$ (recall that $\mathcal{X}$ is
compact), then we can partition $\mathcal{Z}$ as
$2\mathcal{N}(\gamma/2, \mathcal{X}, \|\cdot\|_2)$ sets, such that
if $(y_1, x_1)$ and $(y_2, x_2)$ belongs to the same set, then
$y_1=y_2$ and $\|x_1-x_2\|_2 \leq \gamma/2$.

Further observe that if $y_1=y_2$ and $\|x_1-x_2\|_2 \leq \gamma/2$,
then
\begin{equation*}\begin{split}&
|l\big((w^*,d^*) ,z_1\big)-l\big((w^*,d^*),
z_2)\big)|\\&=\left|[1-y_1(\langle w^*,\,\phi(x_1)
\rangle+d^*)]^+-[1-y_2(\langle w^*,\,\phi(x_2)
\rangle+d^*)]^+\right|
\\ &\leq
\left| \langle w^*,\, \phi(x_1)-\phi(x_2)\rangle \right|\\
&\leq \|w^*\|_{\mathcal{H}} \sqrt{\langle \phi(x_1)-\phi(x_2),
\phi(x_1)-\phi(x_2)\rangle}\\ &\leq
\sqrt{f_{\mathcal{H}}(\gamma)/c}.
\end{split}\end{equation*}
Here the last inequality follows from the definition of
$f_{\mathcal{H}}$.  Hence, the example holds by
Theorem~\ref{thm.nearby}.

\subsection{Proof of Example~\ref{exm.lasso}}\label{app.lasso} It suffices to show the
following lemma, which establish that loss of Lasso solution is
Liptschitz continuous.
\begin{lemma}\label{lem.lasso}
If $w^*(\mathbf{s})$ is the  solution of Lasso given training set
$\mathbf{s}$, then
\[\left|l(w^*(\mathbf{s}), z_a)-l(w^*(\mathbf{s}), z_b)\right| \leq \big[\frac{1}{nc}\sum_{i=1}^n{s_{i|y}}^2+1\big] \|z_a-z_b\|_{\infty}.\]
\end{lemma}
\begin{proof}For succinctness we let $y_i=s_{i|y}$, $x_i=s_{i|x}$
for $i=1,\cdots, n$. Similarly, we let $y_{a(b)}=z_{a(b)|y}$ and
$x_{a(b)}=z_{a(b)|x}$.
 Since $w^*(\mathbf{s})$ is the solution of Lasso, we have (due to optimality)
\[\frac{1}{n}\sum_{i=1}^n (y_i -x_i^\top
w^*(\mathbf{s}))^2+c\|w^*(\mathbf{s})\|_1 \leq
\frac{1}{n}\sum_{i=1}^n (y_i -x_i^\top
0)^2+c\|0\|_1=\frac{1}{n}\sum_{i=1}^n{y_i}^2,\] which implies
$\|w^*\|_1\leq \frac{1}{nc}\sum_{i=1}^n{y_i}^2$. Therefore,
\begin{equation*}\begin{split}\left|l(w^*(\mathbf{s}), z_a)-l(w^*(\mathbf{s}), z_b)\right|
=&\left||y_a -w^*(\mathbf{s})
x_a|-|y_b-w^*(\mathbf{s})x_b|\right|\\
\leq &\left|(y_a -w^*(\mathbf{s})
x_a)-(y_b-w^*(\mathbf{s})x_b)\right|\\
\leq &|y_a-y_b|+\|w^*(\mathbf{s})\|_1
\|x_a-x_b\|_{\infty}\\
\leq &(\|w^*(\mathbf{s})\|_1+1)\|z_a-z_b\|_{\infty}
\\ =&\big[\frac{1}{nc}\sum_{i=1}^n{y_i}^2+1\big]
\|z_a-z_b\|_{\infty}.
\end{split}\end{equation*}Here the first two inequalities holds from
triangular inequality, and the last inequality holds due to $z=(x,
y)$.
\end{proof}

\subsection{Proof of Example~\ref{exm.NN}}\label{app.nn} To see why the example
holds,  it suffices to show the following lemma, which establishes
that the neural network mentioned is Lipschitz continuous. For
simplicity, we write the prediction given  $x\in \mathcal{X}$ as
$NN(x)$.
\begin{lemma}\label{lem.nn}
 Fixed $\alpha, \beta$, if a $d$-layer neural network satisfying that
$|\sigma(a)-\sigma(b)|\leq \beta |a-b|$, and $\sum_{j=1}^{N_v}
|w^v_{ij}|\leq \alpha$ for all $v, i$, then the following holds:
\[|l(A_{\mathbf{s}}, z)-l(A_{\mathbf{s}}, \hat{z})| \leq (1+\alpha^d\beta^d)\|z-\hat{z}\|_{\infty}.\]
\end{lemma}
\begin{proof} Let $x_i^v$ and $\hat{x}_i^v$ be the output of the
$i^{th}$ unit of the $v^{th}$ layer for samples $z$ and $\hat{z}$
respectively. Let $\mathbf{x}^v$ and $\hat{\mathbf{x}}^v$ be the
vector such that the $i^{th}$ elements are $x_i^v$ and $\hat{x}_i^v$
respectively. From $\sum_{i=1}^{N_v} |w^v_i|\leq \alpha$ we have
\begin{equation*}
\begin{split}
|x_i^{v}-\hat{x}_i^{v}| & = \left|\sigma(\sum_{j=1}^{N_v}w^v_{ij}
x_i^{v-1})
-\sigma(\sum_{j=1}^{N_v}w^v_{ij} \hat{x}_j^{v-1})\right|\\
& \leq \beta  \left|\sum_{j=1}^{N_v}w^v_{ij} x_i^{v-1}
-\sum_{j=1}^{N_v}w^v_{ij} \hat{x}_j^{v-1} \right|\\
&\leq \beta \alpha
\|\mathbf{x}^{v-1}-\hat{\mathbf{x}}^{v-1}\|_{\infty}.
\end{split}\end{equation*}
Here, the first inequality holds from the Lipschitz condition of
$\sigma$, and the second inequality holds from $\sum_{j=1}^{N_v}
|w^v_{ij}|\leq \alpha$. Iterating over $d$ layers, we have
\[|NN(z_{|x})-NN(\hat{z}_{|x})|=|x^d-\hat{x}^d| \leq \alpha^d\beta^d
\|\mathbf{x}-\hat{\mathbf{x}}\|_\infty,\] which implies
\begin{equation*}\begin{split}|l(A_{\mathbf{s}}, z)-l(A_{\mathbf{s}},
\hat{z})|
=&\left||z_{|y} -NN(z_{|x})|-|\hat{z}_{|y} -NN(\hat{z}_{|x})|\right|\\
\leq&\|z_{|y}-\hat{z}_{|y}|+ |NN(z_{|x}) -NN(\hat{z}_{|x})|\\ \leq &
(1+\alpha^d\beta^d)\|z-\hat{z}\|_{\infty}.
\end{split}\end{equation*}This proves the lemma.
\end{proof}

\subsection{Proof of Example~\ref{exm.PCA}}\label{app.PCA}
We show that the loss to PCA is Lipschitz continuous, and then apply
Theorem~\ref{thm.nearby}.

 Let $(w^*_1(\mathbf{s}),\cdots, w^*_d(\mathbf{s}))$ be the solution of PCA
trained on $\mathbf{s}$. Thus we have
\begin{equation*}\begin{split}
&\left|l((w^*_1(\mathbf{s}),\cdots, w^*_d(\mathbf{s})), z_a)-l((w^*_1(\mathbf{s}),\cdots, w^*_d(\mathbf{s})), z_b)\right|\\
=&\left|\sum_{k=1}^d(w_k^*(\mathbf{s})^\top
z_a)^2-\sum_{k=1}^d(w_k^*(\mathbf{s})^\top
z_b)^2\right|\\
\leq & \sum_{k=1}^d \left|[w_k^*(\mathbf{s})^\top
z_a-w_k^*(\mathbf{s})^\top
z_b][w_k^*(\mathbf{s})^\top z_a+w_k^*(\mathbf{s})^\top z_b]\right|\\
\leq &2dB\|z_a-z_b\|_2,
\end{split}\end{equation*}where the last inequality holds because
$\|w_k^*(\mathbf{s})\|_2=1$ and $\|z_a\|, \|z_b\| \leq B$. Hence,
the example holds by Theorem~\ref{thm.nearby}.

\subsection{Proof of Example~\ref{exm.margin}}
Set $\hat{\mathbf{s}}$ as
\[\hat{\mathbf{s}}\triangleq \{s_i\in \mathbf{s}|\mathcal{D}(s_i,
\mathcal{A}_\mathbf{s})>\gamma\}.\]And let $c_1, \cdots,
c_{\mathcal{N}(\gamma/2,\mathcal{X}, \rho)}$ be a $\gamma/2$ cover
of $\mathcal{X}$. Thus, we can partition $\mathcal{Z}$ to
$2{\mathcal{N}(\gamma/2,\mathcal{X}, \rho)}$ subsets $\{C_i\}$, such
that if \[z_1, z_2 \in C_i;\quad\Longrightarrow\quad
y_1=y_2;\,\,\&\,\rho(x_1,x_2) \leq \gamma.\] This implies that:
\[z_1 \in \hat{\mathbf{s}},\, z_1, z_2\in C_i;\quad\Longrightarrow\quad y_1=y_2;\,\,\mathcal{A}_\mathbf{s}(x_1)=\mathcal{A}_\mathbf{s}(x_2);\quad\Longrightarrow\quad l(\mathcal{A}_{\mathbf{s}}, z_1)=l(\mathcal{A}_{\mathbf{s}}, z_2).\]
By definition, $\mathcal{A}$ is
$(2{\mathcal{N}(\gamma/2,\mathcal{X}, \rho)}, 0, \hat{n})$ pseudo
robust.

{\small
\bibliographystyle{plain}
\bibliography{Phd1}
}

\end{document}

%% file: seteps.tex
\def\centereps#1#2#3{\vskip#2\relax\centerline{\hbox to#1{\special
  {eps:#3 x=#1, y=#2}\hfil}}}

%% file: setwmf.tex
\def\centerwmf#1#2#3{\vskip#2\relax\centerline{\hbox to#1{\special
  {wmf:#3 x=#1, y=#2}\hfil}}}